\documentclass[sigconf]{acmart}

\usepackage{subcaption}
\usepackage{colortbl}
\usepackage{makecell}
\usepackage{microtype}

\AtBeginDocument{%
  }

\copyrightyear{2026}
\acmYear{2026}
\setcopyright{cc}
\setcctype{by}
\acmConference[KDD '26]{Proceedings of the 32nd ACM SIGKDD Conference on Knowledge Discovery and Data Mining V.2}{August 09--13, 2026}{Jeju Island, Republic of Korea}
\acmBooktitle{Proceedings of the 32nd ACM SIGKDD Conference on Knowledge Discovery and Data Mining V.2 (KDD '26), August 09--13, 2026, Jeju Island, Republic of Korea}
\acmDOI{10.1145/3770855.3818087}
\acmISBN{979-8-4007-2259-2/2026/08}

\begin{document}
\newcommand{\E}{\mathbb{E}}
\newcommand{\Var}{\mathrm{Var}}
\newcommand{\Law}{\mathcal{L}}
\newcommand{\sd}{\mathrm{sd}}

\counterwithout{theorem}{section}
\renewcommand{\thetheorem}{\arabic{theorem}}

\newcommand{\Gap}{\mathrm{Gap}}

\title{Expectations vs. Realities: The Cost of MSE-Optimal Forecasting Under Conditional Uncertainty}

\author{Riku Green}
\email{riku.green@bristol.ac.uk}
\affiliation{%
  \institution{The University of Bristol}
  \city{Bristol}
  \country{UK}
}

\author{Zahraa S. Abdallah}
\email{zahraa.abdallah@bristol.ac.uk}
\affiliation{%
  \institution{The University of Bristol}
  \city{Bristol}
  \country{UK}
}

\author{Telmo M Silva Filho}
\email{telmo.silvafilho@bristol.ac.uk}
\affiliation{%
  \institution{The University of Bristol}
  \city{Bristol}
  \country{UK}
}

\begin{abstract}
Multi-step time series forecasting (MSF) is commonly evaluated using point-wise error metrics such as mean squared error (MSE), implicitly treating the conditional mean as a sufficient target. We show that this can be misleading under conditional uncertainty, where the conditional expectation becomes unrepresentative of typical realized values at longer horizons.
We formalize this effect through a conditional uncertainty gap and prove that whenever this gap is nonzero, no deterministic predictor can simultaneously minimize MSE and match the marginal distribution of realized futures. This establishes a fundamental, model-agnostic trade-off between point accuracy and marginal realism in MSF evaluation.
Using controlled stochastic dynamical systems and nine real-world forecasting benchmarks, we empirically characterize the resulting accuracy--realism frontier and \textbf{quantify the practical cost of MSE-only model selection}. As conditional uncertainty increases with forecast horizon, the attainable set expands into a pronounced Pareto front, separating MSE-optimal but under-dispersed predictors from methods that trade accuracy for realistic \emph{marginal variability}. \textbf{Across benchmarks, we find that small relaxations in MSE ($\boldsymbol{\le 5\%}$) frequently unlock disproportionate gains in marginal realism, with median improvements of $\mathbf{17.3\%}$ and gains exceeding $\mathbf{30\%}$ in some datasets.}
We further show that common forecasting strategies systematically occupy different regions of this frontier: direct multi-output predictors concentrate near the accuracy-optimal extreme, while recursive strategies and sample-based inference favors marginal realism. Together, these results expose a structural failure mode of MSE-based evaluation in long-horizon forecasting and recast strategy and inference selection as navigation of an unavoidable accuracy--realism trade-off.
\end{abstract}

\ccsdesc[500]{Computing methodologies~Machine learning}

\keywords{Time series forecasting, 
multi-step prediction, 
conditional uncertainty, 
evaluation metrics, 
data mining
}

\maketitle

\section{Introduction}
\label{sec:intro}
\begin{figure*}[th]
    \centering
    \begin{subfigure}[b]{0.32\linewidth}
        \centering
        \includegraphics[width=\linewidth]{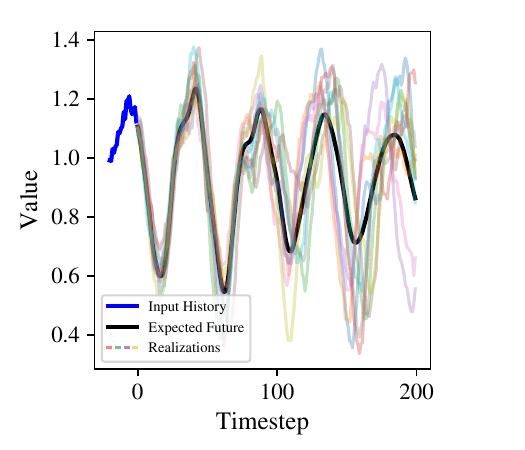}

        \caption{Expectation vs Realizations}
        \label{subfig:gapcurves_a}
    \end{subfigure}
    \hfill
    \begin{subfigure}[b]{0.32\linewidth}
        \centering
        \includegraphics[width=\linewidth]{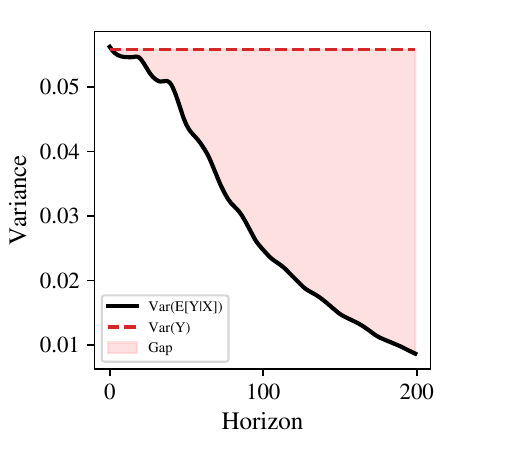}

        \caption{Bayes predictor under-dispersion}
        \label{subfig:gapcurves_b}
    \end{subfigure}
    \hfill
    \begin{subfigure}[b]{0.32\linewidth}
        \centering
        \includegraphics[width=\linewidth]{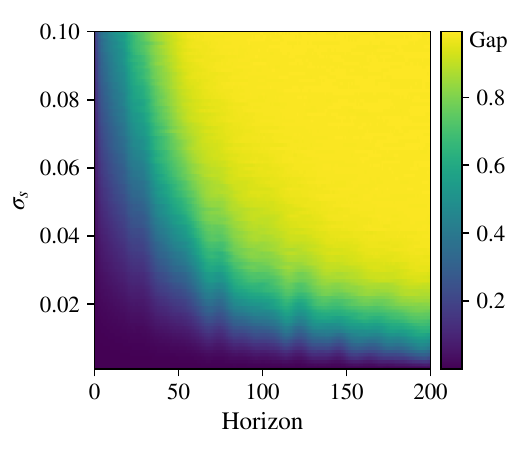}
        \caption{Conditional uncertainty gap}
        \label{subfig:gapcurves_c}
    \end{subfigure}

    \caption{Visualizing the forecasting gap. (a): As the horizon increases, realized trajectories (colored) diverge from the smooth conditional expectation (black), motivating how the ``average'' future can become unrepresentative of typical realized outcomes. (b): The MSE-optimal predictions computed over multiple instances fail to match the target's marginal variability across instances: the predictive marginal variance (black) drops below the true marginal variance (red dashed), yielding a widening variance gap. (c): Increasing stochasticity ($\sigma_s$) induces larger conditional uncertainty, and the gap grows with horizon.}

    \label{fig:gap_curves}
    \Description[Forecasting gap increases with horizon and stochasticity]{Three-panel figure visualizing the forecasting gap. The first panel shows realized trajectories diverging from a smooth conditional expectation as the forecast horizon increases. The second panel shows the MSE-optimal predictor having lower marginal variance than the true target distribution. The third panel is a heatmap showing that the conditional uncertainty gap increases with both stochasticity and forecast horizon.}
\end{figure*}

Multi-step time series forecasting (MSF) is a fundamental problem in machine learning and statistics, with applications in domains such as energy systems, finance, traffic prediction, and climate modeling  \cite{kong2025deep}. Given a window of past observations, the task is to predict a sequence of future values over a fixed horizon. Recent benchmarking typically assesses performance using point-wise error metrics such as mean squared error (MSE) \cite{li2025tsfm_bench, qiu2024tfb_bench, hu2025fintsb_bench, wang2024deep_bench, shao2024exploring_bench, green2025stratify}, implicitly treating the conditional expectation of the future as the desired prediction target.

However, this assumption becomes fragile as the forecast horizon increases. In many real-world systems, uncertainty accumulates over time \cite{bontempi2012machine}, and the conditional distribution of future outcomes given the past becomes increasingly dispersed. This occurs whenever realizations diverge in qualitatively different ways from the same initial condition, as illustrated in Figure~\ref{subfig:gapcurves_a}. As the horizon grows, the expectation--realization gap widens (Figures~\ref{subfig:gapcurves_b} and~\ref{subfig:gapcurves_c}), and the conditional expectation---the target induced by MSE---need not correspond to a typical realized outcome \cite{le2020probabilistic}. Empirical benchmark studies \cite{zhang2024probts_bench} confirm this effect: MSE-optimized forecasts often become increasingly under-dispersed at long horizons. This suggests a tension between forecasting conditional expectations and producing forecasts that are distributionally representative of realized futures.

To formalize this effect, we introduce a conditional uncertainty gap that measures the fraction of future variability that is irreducible given the past. When this gap is nonzero, the conditional expectation becomes unrepresentative of typical realizations. We show that in such regimes, no deterministic predictor can simultaneously minimize mean squared error and match the \emph{marginal} distribution of true future values (across instances at horizon $h$). In other words, pointwise accuracy and marginal realism become fundamentally incompatible objectives. This induces an unavoidable Pareto trade-off \cite{miettinen1999nonlinear} between approximating conditional expectations and reproducing the distribution of realized futures. Crucially, this trade-off is independent of model class, training procedure, or architectural choice: it is a structural property of multi-step prediction under conditional uncertainty. As a consequence, evaluation protocols that select models solely by MSE can systematically favor under-dispersed forecasts, even when alternatives with nearly identical error better match the distribution of realized futures.

This perspective reveals that common forecasting strategies implicitly select different operating points on this trade-off. Direct multi-output predictors, aligned with estimating conditional expectations under squared loss \cite{taieb2014machine}, concentrate near the MSE-optimal extreme and often produce smoothed, under-dispersed forecasts. Recursive strategies, by contrast, propagate variability through time and therefore \emph{can} generate trajectories that more closely resemble typical realizations, but at the cost of higher pointwise error. Probabilistic forecasting models do not remove this tension: different inference rules (e.g., predictive mean or single-sample rollout) simply occupy different regions of the same accuracy--realism frontier.

Our findings support the view that multi-step forecasting cannot be fully characterized by pointwise error alone \cite{zhang2024probts_bench}. Squared loss elicits the conditional mean, but under conditional uncertainty this functional can become unrepresentative of typical realized futures.
This helps explain why recursive, direct, and probabilistic inference rules often dominate in different regimes \cite{taieb2014machine, green2025epistemic, green2025stratify}: each strategy prioritizes a different functional of the future distribution, and whether these priorities align depends on the underlying data-generating process. When the conditional uncertainty gap is small, MSE-optimality and realism are largely aligned; as the gap grows, they separate into an accuracy--realism frontier. Across real-world benchmarks, a modest $5\%$ increase in MSE frequently yields large gains in marginal realism, with a median improvement of $17\%$.

Our goal is not to propose a new forecasting architecture, but to turn this trade-off into a practical model-selection protocol. We recommend keeping MSE as the primary constraint, then using inexpensive distributional diagnostics such as marginal $W_1$ to choose among near-MSE-optimal candidates.

Our contributions are threefold:
\begin{itemize}
    \item \textbf{A theoretical explanation of over-smoothing in long-horizon forecasting.}
    We prove that no deterministic predictor can minimize mean squared error and match the distribution of realized futures. This explains a common long-horizon failure mode: MSE-optimal predictors are necessarily under-dispersed, not because of model misspecification but due to irreducible conditional uncertainty.

    \item \textbf{Quantifying the practical accuracy--realism trade-off.}
    We empirically characterize the resulting accuracy--realism Pareto frontier and show that it is readily attainable in practice. Across synthetic systems and real-world benchmarks, small relaxations in MSE often yield disproportionately large gains in marginal realism, demonstrating that strict MSE-based model selection can incur substantial and avoidable realism costs.

    \item \textbf{Forecasting strategy as a practical control knob.}
    While probabilistic sampling naturally navigates the accuracy--realism trade-off, we show empirically that the \emph{multi-step strategy} itself (direct versus recursive prediction) provides a simple and effective mechanism for moving along the Pareto frontier. This reframes strategy selection as a concrete design choice for trading pointwise error against distributional representativeness, independent of model class.
\end{itemize}


\section{Conditional Uncertainty in MSF}
\label{sec:msf_condUnc}

This section formalizes how uncertainty accumulates in multi-step forecasting and why the conditional mean can cease to be representative of \emph{realized} futures. We aim to characterize intrinsic uncertainty in $Y_h \mid \mathbf{X}$ before our main accuracy--realism result.

\paragraph{Multi-Step Forecasting Setup}

Let $\{y_t\}_{t \in \mathbb{Z}}$ be a real-valued stochastic time series with bounded second moments,
\begin{equation}
\sup_t |y_t| < \infty, 
\qquad 
\mathrm{Var}(y_t) < \infty .
\end{equation}
For simplicity, we assume weak stationarity so that $\mathbb{E}[y_t]$ and $\mathrm{Var}(y_t)$ are time-invariant.
Given a window length $L$, define the past context and the $h$-step-ahead target as
\begin{equation}\mathbf{X} = (y_{t-L}, \ldots, y_{t-1}) \in \mathbb{R}^L, \quad Y_h = y_{t+h}.\end{equation}
Multi-step forecasting seeks a predictor $f_h : \mathbb{R}^L \to \mathbb{R}$ such that $f_h(\mathbf{X})$ approximates $Y_h$.

For a fixed horizon $h$, $Y_h$ is a random variable conditional on the past $\mathbf{X}$. By the law of total variance,
\begin{equation}
\mathrm{Var}(Y_h) 
= \mathrm{Var}\!\left( \mathbb{E}[Y_h \mid \mathbf{X}] \right)
+ \mathbb{E}\!\left[ \mathrm{Var}(Y_h \mid \mathbf{X}) \right].
\end{equation}

The first term, $\mathrm{Var}(\mathbb{E}[Y_h \mid \mathbf{X}])$, quantifies how much the conditional mean varies across different pasts.
The $\mathbb{E}[\mathrm{Var}(Y_h \mid \mathbf{X})]$ term measures the average spread of future outcomes given the same past.
As the forecast horizon $h$ increases, $\mathbb{E}[\mathrm{Var}(Y_h \mid \mathbf{X})]$ typically increases due to noise \cite{taieb2014machine}. This can occur when the unconditional variance of $Y_h$ remains bounded, reflecting the accumulation of uncertainty over time.


We define a horizon-dependent signal-to-noise ratio
\begin{equation}
\mathrm{SNR}_h 
= 
\frac{\mathrm{Var}(\mathbb{E}[Y_h \mid \mathbf{X}])}
{\mathrm{Var}(\mathbb{E}[Y_h \mid \mathbf{X}]) 
+ \mathbb{E}[\mathrm{Var}(Y_h \mid \mathbf{X})]}.
\end{equation}
This quantity measures the fraction of total variability explained by systematic dependence on the past.

Equivalently, we define the \emph{conditional uncertainty gap}
\begin{equation}
\mathrm{Gap}_h 
= 1 - \mathrm{SNR}_h
=
\frac{\mathbb{E}[\mathrm{Var}(Y_h \mid \mathbf{X})]}
{\mathrm{Var}(\mathbb{E}[Y_h \mid \mathbf{X}]) 
+ \mathbb{E}[\mathrm{Var}(Y_h \mid \mathbf{X})]}.
\end{equation}

When $\mathrm{Gap}_h$ is small, realized values at horizon $h$ are tightly concentrated around $\mathbb{E}[Y_h \mid \mathbf{X}]$, so the conditional mean is representative.
When $\mathrm{Gap}_h$ is large, irreducible uncertainty dominates: $\mathbb{E}[Y_h \mid \mathbf{X}]$ explains little of the variance of realized $Y_h$, and even a perfect estimator of the conditional mean can yield forecasts that under-represent the marginal variability of realized futures.

\begin{figure*}[t]
    \centering
    \begin{subfigure}[b]{0.32\linewidth}
        \centering
        \includegraphics[width=\linewidth]{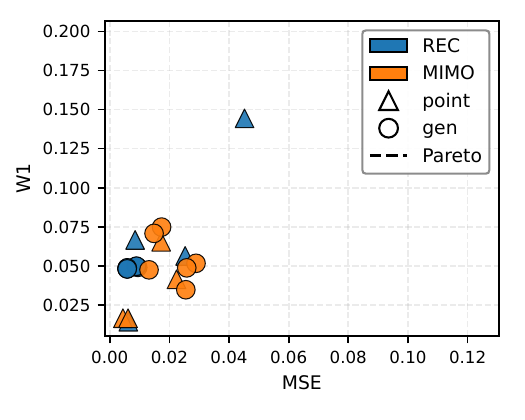}
        \caption{h=20, gap=0.05}
        \label{lowgap}
    \end{subfigure}
    \hfill
    \begin{subfigure}[b]{0.32\linewidth}
        \centering
        \includegraphics[width=\linewidth]{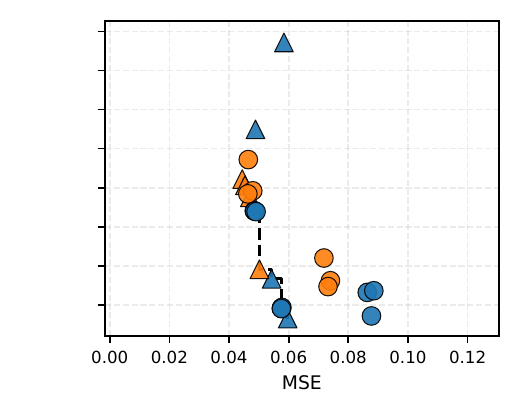}
        \caption{h=100, gap=0.57}
        \label{midgap}
    \end{subfigure}
    \hfill
    \begin{subfigure}[b]{0.32\linewidth}
        \centering
        \includegraphics[width=\linewidth]{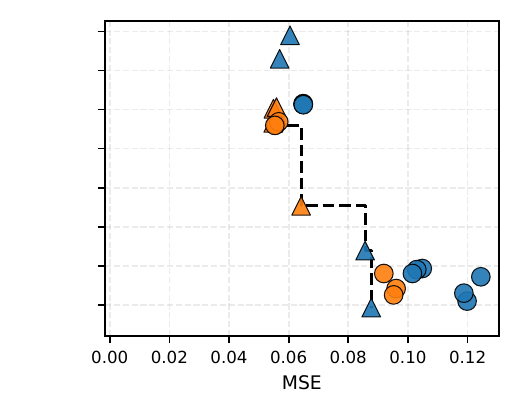}
        \caption{h=200, gap=0.87}
        \label{highgap}
    \end{subfigure}
    \caption{
Attainable accuracy--realism trade-offs under low and high conditional uncertainty.
MSE versus marginal $W_1$  distance for a representative low-gap regime (left) and high-gap regime (right) in the stochastic Mackey--Glass \cite{mackey1977oscillation} system.
Each point corresponds to a trained forecasting strategy.
When the gap is small, methods cluster near a single operating point.
As the gap increases, the attained set of performances expands into a clear Pareto frontier, separating MSE-optimal but under-dispersed predictors from methods that trade accuracy for improved marginal realism.}

\label{fig:pareto_low_high}
\Description[Pareto front widens as conditional uncertainty increases]{Three scatter plots comparing MSE and marginal Wasserstein distance at horizons 20, 100, and 200. At horizon 20, the points cluster near one operating region. At horizons 100 and 200, the points spread into a clearer Pareto frontier, showing a growing trade-off between point accuracy and marginal realism.}
\end{figure*}

\section{Evaluation-Induced Accuracy--Realism Trade-offs}
\label{sec:theory}

MSF evaluation practice primarily treats the conditional mean as the sole prediction target and underlies much of the empirical comparisons made in the literature
\cite{li2025tsfm_bench, wang2024deep_bench, green2025stratify, hu2025fintsb_bench}.
At the same time, empirical studies have repeatedly observed that MSE-optimal forecasts can appear overly smooth or unrepresentative of typical realized futures, particularly at longer horizons
\cite{le2020probabilistic}.

In this section, we show that this phenomenon is not an artifact of model capacity, optimization, or forecasting strategy, but instead reflects a structural limitation of \emph{standard multi-step forecasting evaluation}.
Our goal is to \emph{formalize when and why commonly used evaluation objectives implicitly force a trade-off} between pointwise accuracy and distributional realism.
We focus on deterministic point forecasters and compare pointwise accuracy, measured by MSE, to \emph{marginal} distributional realism measured by the Wasserstein distance $W_1$ \cite{su2017order_wasser}.
This choice reflects both the structure of standard evaluation protocols and the mechanism underlying over-smoothing.

Under typical MSF evaluation, each input context $\mathbf{X}$ is associated with a single realized future $Y_h$, and each deterministic forecaster produces a single prediction $f_h(\mathbf{X})$.
As a result, the conditional distribution $p(Y_h \mid \mathbf{X})$ is not directly observable, and proper scoring rules such as CRPS \cite{gneiting2007strictly, jordan2019evaluating_scoring_forecast} cannot be applied to deterministic forecasts.
The only distributional information consistently identifiable from standard evaluation data is the \emph{marginal distribution across contexts}, obtained from the samples $\{f_h(\mathbf{X}_i)\}_{i=1}^N$ and $\{Y_{h,i}\}_{i=1}^N$.
For this reason, marginal distribution matching represents the \emph{weakest nontrivial notion of realism} that can be evaluated under standard deterministic MSF protocols.
Importantly, it is also directly aligned with the mechanism behind over-smoothing: the collapse of marginal variance induced by MSE-optimal predictors.

We emphasize that marginal realism does not guarantee full trajectory realism.
However, failure at the marginal level necessarily implies failure under any stronger notion of realism.
Since any trajectory-level or conditional distributional match implies equality of the corresponding horizon-wise marginals by projection, failure to match the marginal distribution at any horizon precludes these stronger notions of realism (Appendix~\ref{app:additional_theory}).
We further show that analogous incompatibilities arise at the trajectory level and for stochastic predictors under squared loss.
Thus, the phenomenon is not specific to marginal $W_1$ nor to deterministic forecasting.

\paragraph{Setup.}
Let $\mathcal{L}(Z)$ denote the (marginal) law of a random variable $Z$.
Let $f_h(\mathbf{X})$ be any deterministic predictor of the future value $Y_h$ given the past $\mathbf{X}$.
We evaluate such predictors using two criteria:
\begin{enumerate}
\item \textbf{Pointwise accuracy}, measured by mean squared error,
\begin{equation}
\mathcal{R}_{\mathrm{MSE}}(f_h)
= \mathbb{E}\!\left[(f_h(\mathbf{X}) - Y_h)^2\right].
\end{equation}
\item \textbf{Distributional realism}, measured by agreement between the marginal laws
$\mathcal{L}(f_h(\mathbf{X}))$ and $\mathcal{L}(Y_h)$.
\end{enumerate}

Under squared loss, the first criterion is uniquely minimized by the conditional expectation
$f_h^\star(\mathbf{X}) = \mathbb{E}[Y_h \mid \mathbf{X}]$.
The second criterion requires that the marginal distribution of predictions match that of realized futures.
Standard evaluation therefore implicitly asks for both properties simultaneously.
The following theorem shows that, whenever conditional uncertainty is non-degenerate, this objective is internally inconsistent.

\begin{theorem}[Accuracy--Realism Trade-off]
\label{thm:tradeoff}
Assume $\mathrm{Var}(Y_h)\in(0,\infty)$ and $\mathrm{Gap}_h>0$.
Then no deterministic predictor $f_h$ can simultaneously
\begin{enumerate}
\item minimize mean squared error,
\item and match the marginal distribution of $Y_h$.
\end{enumerate}
More precisely, if $f_h^\star(\mathbf{X})=\mathbb{E}[Y_h\mid\mathbf{X}]$ denotes the MSE-optimal predictor, then
\begin{equation}
W_1\!\left(\mathcal{L}(f_h^\star(\mathbf{X})),\,\mathcal{L}(Y_h)\right) > 0 .
\end{equation}
Conversely, any deterministic predictor whose marginal distribution satisfies
$\mathcal{L}(f_h(\mathbf{X}))=\mathcal{L}(Y_h)$ must incur strictly larger mean squared error than $f_h^\star$.
\end{theorem}

\begin{proof}[Proof sketch]
Under squared loss, the unique minimizer of pointwise risk is the conditional expectation
$f_h^\star(\mathbf{X})=\mathbb{E}[Y_h\mid\mathbf{X}]$.
By the law of total variance,
\[
\mathrm{Var}(Y_h)
= \mathrm{Var}(\mathbb{E}[Y_h\mid\mathbf{X}])
+ \mathbb{E}[\mathrm{Var}(Y_h\mid\mathbf{X})].
\]
When $\mathrm{Gap}_h>0$, the second term is strictly positive, implying that
$\mathrm{Var}(f_h^\star(\mathbf{X})) < \mathrm{Var}(Y_h)$.
Since the two variables share the same mean, their marginal distributions cannot coincide.
For one-dimensional distributions with finite second moments, this variance mismatch implies a strictly positive $W_1$ distance.
The converse follows because any predictor matching the marginal distribution of $Y_h$ must have variance $\mathrm{Var}(Y_h)$ and therefore cannot coincide almost surely with $f_h^\star(\mathbf{X})$, implying strictly larger MSE.
\end{proof}

\subsection{Implications for Forecasting Objectives}
\label{sec:implications}

\textbf{The trade-off is induced by evaluation, not by the learning algorithm.}
When $\mathrm{Gap}_h>0$, Theorem~\ref{thm:tradeoff} shows that no deterministic forecaster can be both MSE-optimal and marginally realistic under standard evaluation.
This under-dispersion of the conditional mean is therefore not a consequence of limited model capacity, suboptimal optimization, or the choice of multi-step forecasting strategy, but instead reflects the conditional distribution of $Y_h$ itself.

\textbf{Forecast evaluation implicitly selects an operating point.}
Deterministic forecasters therefore lie on an intrinsic accuracy--realism Pareto frontier.
The conditional expectation occupies the MSE-optimal extreme, while moving toward marginal realism necessarily incurs additional pointwise error.
This explains why empirical comparisons that rank models solely by MSE may systematically favor oversmoothed forecasts.
Learning a predictive distribution does not eliminate this tension: deployed forecasts still arise from a deterministic decision rule (e.g., using the mean versus sampling), which selects an operating point on the same frontier.
In Section~\ref{sec:experiments}, we show empirically that this trade-off arises in practice and grows with the conditional uncertainty gap, and that trajectory-level realism metrics exhibit the same qualitative frontier structure.

\section{Empirical Characterization of the Trade-off}
\label{sec:experiments}

Given the intrinsic frontier implied by Section \ref{sec:theory}, our goal is to characterize its empirical manifestation and quantify the practical cost of MSE-only evaluation. Our empirical results support three claims: (i) the attainable accuracy–realism frontier expands with conditional uncertainty; (ii) small relaxations in MSE frequently unlock disproportionate gains in realism; and (iii) forecasting strategies and inference rules systematically populate different regions of this frontier. The remainder of this section supports these claims in controlled systems and real-world benchmarks.

We measure marginal realism using $W_1$ distance, which directly aligns with Theorem~\ref{thm:tradeoff}.
To ensure conclusions are not specific to a single notion of realism, we also report trajectory-level and multivariate realism metrics. Dynamic Time Warping (DTW), Vector Wasserstein ($W_{1-vec}$), order-preserving Wasserstein (OPW), and fused Gromov--Wasserstein (FGW)); additional robustness are provided in Appendices~\ref{app:robustness} 
and \ref{app:what_each_metric}.
These serve as robustness checks; our main empirical claims are stated in terms of marginal $W_1$, which is the weakest distributional notion identifiable under standard deterministic MSF evaluation.

Across datasets, we evaluate the same base predictors under recursive (REC) and direct multi-output (MIMO) strategies.
Our deterministic predictors include Linear Ridge Regression (LR), Multilayer Perceptrons (MLP), and Decision Trees (DT).
We prioritize empirical coverage over state-of-the-art accuracy. To confirm the incompatibility is model-agnostic, we replicate results for Chronos \cite{ansari2024chronos} in Appendix \ref{app:chronos}.
We additionally train conditional normalizing flows \cite{dinh2016density_realnvp} and evaluate them via deterministic decision rules (e.g., predictive mean and single-sample rollouts), allowing us to compare how probabilistic models \emph{navigate} the attainable trade-offs through inference choice.
Implementation details and hyperparameters are deferred to Appendix~\ref{app:data}. 
For each horizon $h \in [1,\dots,200]$, each trained method induces a point in objective space $(\mathrm{MSE}_h, W_{1h})$ (lower is better on both axes).
We summarize the best achievable compromises by the empirical Pareto front over the evaluated methods.

\subsection{Controlled Trade-off Geometry}
We begin with a controlled synthetic system in which conditional uncertainty can be directly measured.
This allows us to relate the geometry of the attainable accuracy--realism frontier to the conditional uncertainty gap defined in Section~\ref{sec:theory}.

We generate time series from a stochastic Mackey--Glass system \cite{mackey1977oscillation} with additive state noise of magnitude $\sigma_s$.
Varying $\sigma_s$ continuously controls the level of conditional uncertainty while keeping the dynamics bounded. 

For each evaluation context $\mathbf{x}_i$, we approximate the conditional distribution $p(Y_h\mid \mathbf{X}=\mathbf{x}_i)$ using $M=1000$ independent stochastic rollouts.
This enables direct estimation of the conditional uncertainty gap ${\mathrm{Gap}}_h$ as well as the evaluation objectives $\mathrm{MSE}_h$ and $W_{1h}$.
Full simulation details are provided in Appendix~\ref{app:synthetic_details}.

Figure~\ref{fig:gap_curves} shows that ${\mathrm{Gap}}_h$ increases monotonically with both forecast horizon $h$ and noise level $\sigma_s$.
The system therefore spans regimes ranging from low conditional uncertainty, where the conditional expectation is representative of typical futures, to high uncertainty, where it is not.
This controlled setting matches the assumptions of the theoretical analysis and provides a reference for interpreting the geometry of empirical trade-offs.

\textbf{Pareto Geometry Under Low and High Gap}.
We explore how the accuracy--realism trade-off from Theorem~\ref{thm:tradeoff} appears in practice.
We focus on the geometry of the \emph{attainable} trade-offs induced by a finite set of trained models.

In low-gap regimes (Figure \ref{lowgap}), methods cluster tightly in objective space.
Here, minimizing MSE also yields low $W_1$ discrepancy, indicating that the conditional expectation remains representative of marginal outcomes and the attainable frontier is effectively degenerate.

As ${\mathrm{Gap}}_h$ increases, the attainable set expands into a clear Pareto frontier.
Methods separate along the accuracy--realism axis: some concentrate near the MSE-optimal extreme by collapsing toward the conditional expectation, while others trade pointwise accuracy for improved marginal realism.
This separation reflects the incompatibility identified in Theorem~\ref{thm:tradeoff} and shows how it materializes under different forecasting approaches.

\textbf{Trade-off geometry grows with conditional uncertainty.}
To summarize the size of the attainable accuracy–realism trade-off, we measure the length of the Pareto front induced by the trained forecasters.
Longer fronts indicate a wider set of non-dominated operating points.
Figure~\ref{fig:parlen_vs_gap} shows that Pareto front length increases monotonically with the estimated gap ${\mathrm{Gap}}_h$ across all realism metrics, indicating an expanding attainable trade-off as conditional uncertainty grows.

\begin{figure}[t]
    \centering
    \includegraphics[width=\linewidth]{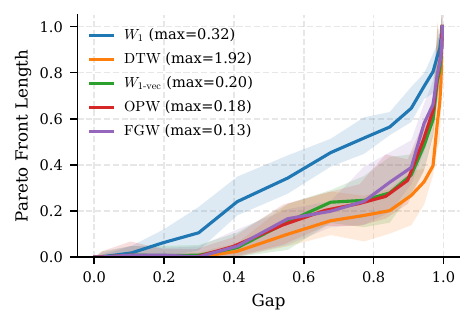}
\caption{
Pareto front length as a function of the estimated conditional uncertainty gap for different realism metrics.
For comparability, front lengths are normalized by the maximum raw length attained by each metric (reported in the legend).
Curves show median values within quantile bins of $\mathrm{Gap}$, with shaded regions indicating interquartile ranges.
All metrics exhibit a monotonic increase in attainable trade-off size with gap, indicating that higher conditional uncertainty induces a wider accuracy--realism frontier.}

\label{fig:parlen_vs_gap}
\Description[Pareto front length grows with uncertainty gap]{Line plot showing normalized Pareto front length as a function of the estimated conditional uncertainty gap for several realism metrics. All curves increase with the gap, with shaded bands showing interquartile ranges, indicating that higher conditional uncertainty creates a wider accuracy--realism frontier.}
\end{figure}

\begin{figure}[t!]
    \centering
    \includegraphics[width=\linewidth]{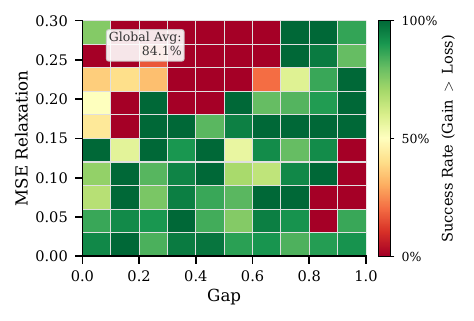}
    \caption{Local trade-offs between MSE relaxation and realism gain when moving from the top two MSE-optimal models in the Pareto front under $W_1$ metric on MG experiments.
    The percentage where relative realism gains ($W_1$) were greater than its MSE relaxation, binned by Gap ($x$-axis) and MSE relaxation ($y$-axis). 
    Colors show the percentage of trials where Gain $>$ Loss. Wider gaps offer wider relaxation success rates.}
    \label{fig:mse_realism_trade}
    \Description[Realism gains often exceed MSE losses]{Heatmap showing the percentage of trials where relative marginal realism gain exceeds MSE relaxation. The horizontal axis bins the conditional uncertainty gap and the vertical axis bins MSE relaxation. High-gap regions are able to show larger success rates for trading MSE losses for realism gains.}
\end{figure}

\subsection{Local Cost of Prioritizing MSE}
\label{sec:local_slope}

The Pareto front characterizes the full set of attainable accuracy--realism trade-offs, but standard evaluation protocols typically operate near the MSE-optimal extreme \cite{li2025tsfm_bench, green2025stratify, shao2024exploring_bench}.
We therefore analyze the \emph{local} behavior of the frontier to quantify how often a small relaxation in MSE yields a net-positive realism trade.

For each noise level and forecast horizon, we extract the Pareto front in the $(\mathrm{MSE}, W_1)$ plane and consider the first two Pareto-optimal models ordered by increasing MSE.
Let $\Delta_{\mathrm{MSE}}$ denote the relative MSE increase required to move from the MSE-optimal model to the next Pareto-optimal alternative, and let $\Delta_{W_1}$ denote the corresponding relative reduction in marginal $W_1$.
We call this step a \emph{successful trade} when $\Delta_{W_1} > \Delta_{\mathrm{MSE}}$.

Figure~\ref{fig:mse_realism_trade} reports that successful trades are prevalent across essentially all gap values (global average $84.1\%$), indicating that even small relaxations in MSE frequently deliver larger relative improvements in marginal realism.

\begin{figure}[t]
    \centering
    \includegraphics[width=\linewidth]{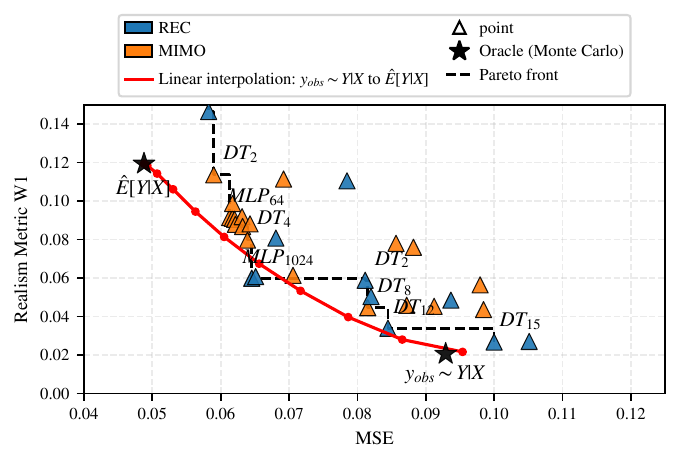}

\caption{
Inference rules navigate the accuracy--realism Pareto frontier.
Oracle prediction performance (stars) obtained via MC rollouts from the true conditional distribution shown as the expected future (mean $\mathbb{E}[Y\mid X]$) and single-sample inference ($Y \sim p(Y\mid X)$).
Linear interpolation (red line) shows that inference rules select operating points along the intrinsic $(\mathrm{MSE}, W_1)$ frontier even in the absence of epistemic error.
Recursive (REC) and direct multi-output (MIMO) point strategies shown in blue and orange respectively. MLP width and decision tree depth indicated by subscripts.
}
\label{fig:prob_forc_moving}
\Description[Inference choices select different frontier points]{Scatter plot of MSE against marginal Wasserstein distance. Oracle mean and oracle sample predictions are shown as stars, with interpolation between them. Recursive and direct multi-output forecasting strategies occupy different regions of the frontier, illustrating how inference rules trade point accuracy against marginal realism.}
\end{figure}

\begin{figure}[t]
    \centering
    \includegraphics[width=\linewidth]{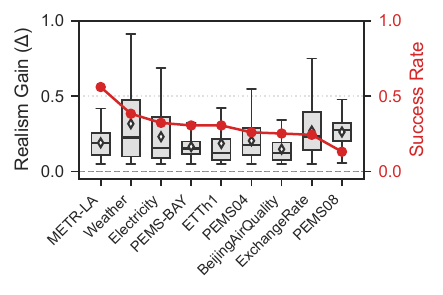}

   \caption{
Real-world accuracy--realism trade-offs under a $5\%$ MSE tolerance.
Boxplots show the distribution of relative marginal realism gains ($W_1$) attainable by Pareto-optimal alternatives whose MSE is within $5\%$ of the MSE-optimal model for each dataset.
Red markers indicate the success rate: the fraction of cases in which the realism gain exceeds the corresponding MSE loss.
Results highlight substantial attainable realism improvements in several datasets, alongside marked heterogeneity across domains.
}
    \label{fig:main_realism_success}
    \Description[Real datasets show realism gains within five percent MSE tolerance]{Boxplots showing relative marginal realism gains available within a five percent MSE tolerance across real-world datasets. Red markers show success rates for each dataset. The figure shows substantial realism improvements for several datasets, with variation across domains.}
\end{figure}

\subsection{Frontier Navigation via MSF Strategies}
\label{sec:strategy_positioning}

Having characterized the attainable accuracy–realism frontier, we now illustrate how common forecasting strategies \emph{navigate} this space.
Figure~\ref{fig:pareto_low_high} shows that direct (MIMO) and recursive (REC) strategies consistently occupy different regions of the frontier, particularly in high-gap regimes.
Direct strategies cluster near the MSE-optimal extreme, while recursive strategies more frequently trade accuracy for improved marginal realism.

Using the same high-gap regime, in Figure~\ref{fig:prob_forc_moving} we further isolate the role of inference using oracle access to the conditional distribution, thereby eliminating modeling and optimization error. Different summaries of the same predictive distribution induce distinct operating points: mean inference concentrates near the MSE-optimal extreme, while single-sample inference trades pointwise accuracy for improved marginal realism. This shows that probabilistic forecasting navigates, rather than resolves, the intrinsic accuracy–realism tension. 
Linear interpolation between these summaries traces a continuous path along the frontier. The same `inference-as-control' effect appears when we used a zero-shot Chronos Time Series Foundation Model \cite{ansari2024chronos} on real world datasets in Appendix \ref{app:chronos}. Theorem~\ref{thm:tradeoff} and Appendix~\ref{app:additional_theory} support these empirical results. 

\label{sec:real_world}
\begin{table*}[t]
    \centering
    \footnotesize 
    
    \begin{subtable}[b]{0.49\textwidth}
        \centering
        \resizebox{\linewidth}{!}{%
        \begin{tabular}{@{}l@{\hspace{4pt}}cccccc@{}}
        \toprule
        Dataset & REC\_P & MIMO\_P & f-REC$_{s}$ & f-MIMO$_{s}$ & f-REC$_{m}$ & f-MIMO$_{m}$ \\ \midrule
        MG                & 0.42 & 0.44 & 0.0 & 0.0 & 0.1  & 0.02 \\
        Electricity       & 0.17 & 0.4  & 0.0 & 0.0 & 0.0  & 0.42 \\
        ExchangeRate      & 0.38 & 0.28 & 0.0 & 0.0 & 0.0  & 0.34 \\
        Weather           & 0.38 & 0.29 & 0.0 & 0.0 & 0.0  & 0.34 \\
        METR-LA           & 0.19 & 0.76 & 0.0 & 0.0 & 0.0  & 0.04 \\
        PEMS-BAY          & 0.01 & 0.93 & 0.0 & 0.0 & 0.0  & 0.06 \\
        PEMS04            & 0.01 & 0.85 & 0.0 & 0.0 & 0.0  & 0.14 \\
        ETTh1             & 0.09 & 0.52 & 0.0 & 0.0 & 0.0  & 0.39 \\
        PEMS08            & 0.01 & 0.98 & 0.0 & 0.0 & 0.0  & 0.01 \\
        BeijingAirQuality & 0.2  & 0.59 & 0.0 & 0.0 & 0.0  & 0.15 \\ \midrule
        \rowcolor[HTML]{EFEFEF} 
        Average           & 0.19 & 0.6  & 0.0 & 0.0 & 0.01 & 0.19 \\ \bottomrule
        \end{tabular}%
        }
        \caption{Fraction of optimal selections under MSE loss.}
        \label{tab:table_mse}
    \end{subtable}
    \hfill 
    \begin{subtable}[b]{0.49\textwidth}
        \centering
        \resizebox{\linewidth}{!}{%
        \begin{tabular}{@{}l@{\hspace{4pt}}cccccc@{}}
        \toprule
        Dataset & REC\_P & MIMO\_P & f-REC$_{s}$ & f-MIMO$_{s}$ & f-REC$_{m}$ & f-MIMO$_{m}$ \\ \midrule
        MG                & 0.12 & 0.15 & 0.22 & 0.49 & 0.0 & 0.0  \\
        Electricity       & 0.01 & 0.1  & 0.02 & 0.77 & 0.0 & 0.09 \\
        ExchangeRate      & 0.57 & 0.2  & 0.0  & 0.16 & 0.0 & 0.08 \\
        Weather           & 0.92 & 0.04 & 0.0  & 0.0  & 0.0 & 0.04 \\
        METR-LA           & 0.7  & 0.0  & 0.0  & 0.29 & 0.0 & 0.0  \\
        PEMS-BAY          & 0.5  & 0.01 & 0.01 & 0.48 & 0.0 & 0.0  \\
        PEMS04            & 0.17 & 0.01 & 0.0  & 0.81 & 0.0 & 0.0  \\
        ETTh1             & 0.14 & 0.06 & 0.01 & 0.76 & 0.0 & 0.0  \\
        PEMS08            & 0.13 & 0.0  & 0.0  & 0.86 & 0.0 & 0.01 \\
        BeijingAirQuality & 0.68 & 0.04 & 0.04 & 0.23 & 0.0 & 0.0  \\ \midrule
        \rowcolor[HTML]{EFEFEF} 
        Average           & 0.39 & 0.06 & 0.03 & 0.49 & 0.0 & 0.02 \\ \bottomrule
        \end{tabular}%
        }
        \caption{Fraction of optimal selections under $W_1$  marginal loss.}
        \label{tab:table_w1}
    \end{subtable}
    
\caption{Comparison of REC and MIMO strategies on real-world datasets. Values indicate the frequency with which each method is optimal for the respective metric (a) for mse and (b) for marginal realism. Subscripts $s$ and $m$ correspond to sample-based and mean-based flows; ``--P'' denotes point methods and ``f--'' denotes generative methods.}
    \label{tab:main_comparison_table}
\end{table*}

\subsection{Opportunity Cost of MSE-Only Model Selection on Real Benchmarks}
\label{sec:real_world_cost}

We evaluate 9 real-world forecasting benchmarks from \cite{shao2024exploring_bench}, spanning energy demand, traffic, weather, air quality, and finance.
For each dataset, we treat the first 7 variates as independent univariate series and evaluate a forecast horizon of 200 steps.
This yields $9 \times 7 \times 200 = 12{,}600$ horizon-specific evaluations (Pareto fronts). Dataset details appear in Appendix~\ref{app:data}.

\textbf{Price of ``MSE-only'' evaluations.}
To quantify the practical cost of strict MSE-based selection, we fix a tolerance $\epsilon=5\%$ around the MSE-optimal model.
For each horizon, we identify near-ties whose MSE lies within $(1+\epsilon)\%$ of the optimum and record the best attainable improvement in marginal realism.
We call this a \emph{successful opportunity} when the realism gain exceeds the MSE loss (i.e., $\Delta \mathrm{Realism} > \Delta \mathrm{MSE}$).
This yields, for each dataset, a distribution of attainable realism improvements and a success probability.

Figure~\ref{fig:main_realism_success} reports these quantities across datasets.
Across all nine datasets, under a $5\%$ MSE tolerance, the median attainable marginal realism improvement is \textbf{17.3\%}, with a median success probability of \textbf{30.6\%}.
This shows that strict MSE-only selection can systematically discard substantially more realistic forecasts, even when accuracy differences are small.
At the dataset level, the effect is often larger.
For example, \textbf{Weather} admits a median realism improvement of \textbf{22.6\%} with success probability \textbf{38.4\%}, and \textbf{ExchangeRate} admits \textbf{23.3\%} with success probability \textbf{24.4\%}.
Even when success is less frequent, unrealized gains can remain large: \textbf{PEMS08} admits \textbf{27.7\%} median improvement but succeeds in \textbf{13.2\%} of cases.
As a robustness check, we reproduce the same $\epsilon=5\%$-opportunity analysis using zero-shot Chronos-T5 and observe consistent positive median realism gains with high success rates in Appendix \ref{app:chronos}.

\textbf{Benchmark-wide pattern and robustness.}
The opportunity cost varies across benchmarks, reflecting differences in
conditional uncertainty, model misspecification, and data regime. However,
the effect is not a corner case: every dataset contains cases where
near-MSE-optimal models offer materially better marginal realism. This
indicates that strict single-metric selection can be systematically
misleading even under a small and practically relevant MSE tolerance.

We further test whether this pattern is an artifact of the candidate model
pool. In Appendix~\ref{app:stronger_baselines}, we repeat the same
$5\%$ tolerance-band analysis using a broader set of forecasting models,
including Linear, MLP, LSTM, N-BEATS, PatchTST, and TCN point forecasters,
together with probabilistic variants. The same qualitative opportunity-cost
pattern remains: near-tied alternatives within the MSE band continue to
yield non-trivial marginal-realism gains across the evaluated datasets.
Thus, the observed effect is not specific to the simpler candidate families
used in the main sweep.

For probabilistic forecasts, we also examine whether improved marginal
realism is merely obtained by producing wider, less informative predictive
distributions. Appendix~\ref{app:calibration_width} reports calibration and
interval-width diagnostics for the best-realism candidate inside the same
MSE tolerance band. These diagnostics show heterogeneous behavior rather
than a uniform widening effect, so marginal realism, calibration, and
sharpness should be treated as complementary model-selection criteria rather
than interchangeable objectives.

\textbf{The uncertainty gap is not required for selection.}
In synthetic systems, repeated rollouts allow us to estimate the
conditional-uncertainty gap directly and relate it to the geometry of the
Pareto frontier. In real-world datasets, however, the conditional
distribution $p(Y_h \mid X)$ is not directly observable, so gap estimation is
necessarily approximate. 

Crucially, the practical opportunity-cost analysis does not require
estimating the uncertainty gap. It is computed directly from quantities
available under standard validation: candidate-model predictions,
ground-truth futures, MSE, and empirical realism metrics. The gap explains
why an accuracy--realism frontier should emerge under conditional
uncertainty, but the tolerance-band selection statistic itself remains
well-defined even when the gap is unknown or imperfectly estimated.

\subsection{Strategies at MSE and Realism Extremes}
\label{sec:extreme_membership}
While previous sections examined the geometry of the attainable accuracy--realism frontier, we now ask which forecasting strategies tend to occupy its \emph{extremes}.
Specifically, we analyze how different strategy and inference choices populate the MSE-optimal and realism-optimal ends across the 12,600 Pareto fronts attained (9 datasets × 7 variates × 200 horizons).

For each experimental configuration, we extract the Pareto front in the $(\mathrm{MSE}, W_1)$ plane and identify its two extreme members: the model with lowest MSE and the model with lowest marginal $W_1$.
We then record the forecasting strategy associated with each extreme and compute the proportion of times each strategy occupies these positions across all experiments.
Table~\ref{tab:main_comparison_table} reports these membership ratios, grouped by dataset and averaged across domains.

The results reveal a strong and systematic clustering of Pareto extremes by strategy and inference choice.
At the MSE-optimal extreme, direct MIMO point predictors dominate, accounting for $60\%$ of lowest-MSE solutions on average, followed by MIMO probabilistic models using mean-based inference ($19\%$).
Notably, MIMO probabilistic models using single-sample inference never attain the MSE-optimal position in any experiment.

In contrast, the realism-optimal extreme is dominated by strategies that propagate or inject variability.
Recursive point predictors account for $39\%$ of lowest-$W_1$ solutions on average, while MIMO probabilistic models using single-sample inference account for $49\%$.

These patterns indicate that the extremes of the accuracy--realism frontier are instead strongly determined by the choice of inference mechanism.
Direct or expectation-aligned inference consistently occupies the accuracy-optimal pole, while recursive and sample-based inference concentrate near the realism-optimal pole.
This separation reinforces the view that strategy and inference choices act as mechanisms for navigating a multi-objective forecasting space, rather than simply improving performance along a single axis.

\subsection{Practical Recommendation}
\label{sec:practical_selection}

Our results do not argue against MSE. Pointwise accuracy is often the
primary operational constraint. The problem is stricter: under conditional
uncertainty, models that are effectively tied in MSE can differ substantially
in whether they reproduce the marginal variability of realized futures.

We therefore recommend a tolerance-band selection rule. First, identify the
MSE-optimal model on validation data. Second, retain all candidates whose
MSE lies within an application-dependent tolerance $\epsilon$ of this optimum.
Third, choose among these near-ties using marginal realism, measured here by
$W_1$, as a simple diagnostic for dispersion. For probabilistic
forecasts, this realism check should be paired with calibration and
sharpness diagnostics, rejecting candidates with clearly worse coverage or
unnecessarily wider predictive intervals.

The tolerance $\epsilon$ encodes the downstream cost of sacrificing point
accuracy. When small MSE differences are operationally decisive, $\epsilon$
should be close to zero. When forecasts are used for simulation, scenario
generation, planning, stress testing, or long-horizon decision support,
modest accuracy losses may be acceptable if they yield futures that better
reflect realized variability. In our real-world experiments, even a $5\%$
MSE band often exposes substantially more realistic alternatives.

Thus, marginal realism should not be optimized in isolation, nor should it
replace proper probabilistic evaluation. Its role is more modest and
practical: keep MSE as the primary constraint, then use inexpensive
distributional diagnostics such as marginal $W_1$, together with calibration
and sharpness where available, to select among models that are effectively
tied under pointwise error.

\section{Related Work}
\label{sec:relwork}

\textbf{Multi-step forecasting strategies.}
Classical approaches to multi-step forecasting distinguish between recursive (iterated) and direct strategies, and analyze this choice primarily through a bias--variance trade-off under squared error loss \cite{taieb2014machine, green2025stratify, green2025epistemic, green2024time}. Early theoretical work showed that an oracle one-step predictor is generally biased for multi-step objectives \cite{taieb2014machine}. More recent analyses, however, indicate that task bias does not necessarily translate into model bias in finite-sample regimes \cite{green2025epistemic}. 
Existing empirical comparisons typically rank strategies using mean squared error (MSE), implicitly treating the conditional mean as the sole prediction target \cite{green2025stratify}. As a result, little attention has been paid to how different strategies affect the distributional properties of predicted trajectories or the variability of long-horizon forecasts.

\textbf{Forecasting benchmarks and evaluation protocols.}
Recent benchmarks for time series forecasting remain heavily centered on pointwise error metrics such as MSE or MAE \cite{li2025tsfm_bench, qiu2024tfb_bench, hu2025fintsb_bench, wang2024deep_bench, shao2024exploring_bench, strom2024performance_metrics}. However, this paradigm does not consider potential trade-offs between pointwise accuracy and distributional realism, nor does it examine whether near-optimal models under MSE may differ substantially in their distributional behavior.

\textbf{Distributional forecasting approaches.}
Parallel work on probabilistic forecasting, popularized by DeepAR \cite{salinas2017deepar}, emphasizes estimating predictive distributions rather than single point forecasts, and is commonly evaluated using proper scoring rules such as the CRPS \cite{gneiting2007strictly, jordan2019evaluating_scoring_forecast}. The ProbTS benchmark \cite{zhang2024probts_bench} explicitly evaluates both point and distributional forecasting across horizons, highlighting tensions between short-term probabilistic accuracy and long-term point accuracy. 
Related approaches include quantile forecasting \cite{lim2021temporal_tft}, conformal prediction methods that construct prediction intervals with coverage guarantees \cite{romano2019conformalized_conformal}, and generative models such as diffusion-based forecasters \cite{rasul2021timegrad}. These methods are often presented as alternatives to point forecasting, with the ability to model predictive distributions viewed as a way to mitigate over-smoothing or uncertainty miscalibration.

In contrast, we show that the tension between pointwise accuracy and distributional realism arises from the conditional distribution itself. Once a decision rule is applied—such as reporting the predictive mean, median, or a sampled trajectory—these methods correspond to different operating points along the same accuracy--realism frontier, rather than eliminating the underlying trade-off.

\textbf{Expectation versus typical realizations.}
The mismatch between conditional expectations and representative realizations is well known in other areas of machine learning, particularly in generative modeling, where minimizing squared error can lead to mode averaging and visually unrealistic samples \cite{metz2016unrolled_gan_mode}. In such settings, the conditional mean is often not a meaningful output, and research has therefore shifted toward distributional or perceptual objectives that prioritize realistic samples over pointwise accuracy \cite{theis2015note}. 

In contrast, time-series forecasting has historically remained centered on pointwise error metrics, with the conditional expectation treated as the canonical prediction target \cite{le2020probabilistic}. We instead show that, under nonzero conditional uncertainty, the mismatch between expectations and typical realizations is structurally unavoidable for deterministic predictors evaluated under squared loss, reframing oversmoothing as an intrinsic property of the forecasting objective rather than a limitation of model capacity or training.

\textbf{Positioning of this work.}
In contrast to prior studies, which treat pointwise and distributional evaluation as separate modeling choices, we show that their incompatibility arises from a structural property of the conditional distribution itself. While multi-objective work has explored trade-offs between accuracy and practical costs \cite{navon2021pareto, fischer2024autoxpcr_multi_ob}, it typically assumes only the MSE notion of accuracy. We instead show that, when conditional uncertainty is nonzero, no deterministic predictor can simultaneously minimize MSE and match the marginal distribution of realized futures, establishing an intrinsic trade-off between point accuracy and distributional realism.

\section{Discussion and future work}
\label{sec:discussion}

\textbf{The practical cost of MSE-only evaluation.}
Across nine real-world benchmarks, allowing a modest 5\% relaxation in MSE yields a median improvement of 17.3\% in marginal realism. This indicates that strict MSE-only model selection can systematically discard substantially more realistic forecasts, even when accuracy differences are small. 
This effect is not a corner case: every dataset exhibits regimes in which near-MSE-optimal models provide materially improved marginal distributional behavior. Together with our theoretical results, this suggests that single-metric evaluation can be structurally misleading in long-horizon forecasting. 


\textbf{Strategy and inference as functional selection.}
Recursive and direct strategies are traditionally compared through bias and variance analyses under squared loss. Our results suggest a more general interpretation: strategy and inference choices implicitly select which functional of the conditional future distribution is emphasized. 
Direct multi-output strategies and mean-based inference align with conditional expectations and therefore concentrate near the MSE-optimal extreme. Recursive strategies and sample-based inference propagate variability and more often occupy realism-favoring regions of the frontier. 
This perspective helps explain why empirical dominance between strategies is inconsistent across horizons and datasets \cite{green2025stratify}: they are not optimizing the same single objective \cite{taieb2014machine}, but navigating different regions of a multi-objective forecasting space.

\textbf{Implications for benchmarking and evaluation.}
Current forecasting benchmarks predominantly rank methods using pointwise error metrics such as MSE or MAE. Our results indicate that this practice can obscure meaningful differences in distributional behavior at long horizons. 
A simple practical adjustment is to complement pointwise metrics with a distributional realism metric or to evaluate models within a small MSE tolerance band and compare attainable realism gains. 
More broadly, we advocate reporting accuracy--realism trade-offs—such as using Pareto fronts in forecasting model selection \cite{borchert2022multi_ts_pareto} or dominance relations—rather than collapsing performance into a single scalar metric.

\paragraph{Limitations and future directions.}
Our main theoretical result is stated for deterministic predictors and deterministic inference rules. While probabilistic models can be embedded into this framework through specific decision rules, a full characterization of the accuracy--realism frontier for stochastic predictors remains open.
Our notion of realism is also intentionally limited. Marginal $W_1$ captures a simple and observable failure mode of MSE-optimal forecasting: under-dispersion across realized futures. It should not be interpreted as a sufficient condition for conditional correctness, calibration, or
trajectory-level realism. In probabilistic forecasting settings, marginal realism should therefore be used alongside proper scoring rules, calibration diagnostics, and sharpness measures.
Extending the analysis to trajectory-level objectives and nonstationary settings, such as concept drift, is a natural direction for future work.

\subsection{Conclusion}
\label{sec:conc}
We showed that when conditional uncertainty is nonzero, no deterministic predictor can simultaneously minimize mean squared error and match the marginal distribution of realized futures. This induces a fundamental accuracy--realism trade-off in multi-step forecasting.
Empirically, across nine real-world benchmarks, a 5\% tolerance in MSE yields a median 17.3\% gain in marginal realism, demonstrating a tangible opportunity cost to strict MSE-only model selection.
Our findings suggest that long-horizon forecasting performance is better understood as navigating an intrinsic accuracy--realism frontier rather than optimizing a single pointwise error metric.
The practical implication is not to discard MSE, but to treat long-horizon forecasting under conditional uncertainty as a constrained multi-objective selection problem: keep MSE as the primary constraint, define an application-dependent tolerance band around the MSE optimum, and use inexpensive distributional diagnostics such as marginal $W_1$, together with calibration and sharpness checks where available, to choose among near-tied candidates.

\bibliographystyle{ACM-Reference-Format}
\bibliography{concise_references}


\appendix

\section{Dataset and Experimental Details}
\label{app:data}

\subsection{Real-world benchmarks}
We use the nine real-world datasets from \cite{shao2024exploring_bench}:
ETTh1 ($T{=}14{,}400$), Electricity ($T{=}26{,}304$), METR-LA
($T{=}34{,}272$), PEMS04 ($T{=}16{,}992$), PEMS08 ($T{=}17{,}856$),
PEMS-BAY ($T{=}52{,}116$), Weather ($T{=}52{,}696$),
BeijingAirQuality ($T{=}36{,}000$), and ExchangeRate ($T{=}7{,}588$).
For each dataset, we treat the first seven variates as separate
univariate forecasting tasks.

Each series is split chronologically, with the first $80\%$ used for
training and the remaining $20\%$ for evaluation. All experiments use a
lag window of length $L=20$ and a forecast horizon of $H=200$. We draw
$N_{\mathrm{eval}}=1000$ evaluation windows uniformly from the test
portion of each series. Each univariate series is standardized using the
training-set mean and standard deviation, and all reported metrics are
computed in this normalized space.

\subsection{Forecasting models and inference rules}
The main empirical sweep evaluates deterministic point forecasters under
recursive (REC) and direct multi-output (MIMO) strategies using Linear
Ridge Regression, MLPs, and Decision Trees. We additionally train
conditional normalizing flows under both REC and MIMO formulations and
evaluate them using deterministic inference rules such as predictive
means and single-sample rollouts. All models are trained with fixed
hyperparameters and without dataset-specific tuning: Ridge uses
$\alpha=0.1$, MLPs use one 2048-unit hidden layer with
$\alpha\in\{0,8\}$, decision trees use maximum depth 10, and RealNVP
flows use 12 affine coupling layers with hidden widths
$\{16,64,256\}$. The purpose of the model set is to sample attainable
operating points on the accuracy--realism frontier rather than to
exhaustively optimize each benchmark.

\subsection{Synthetic Mackey--Glass setting}
\label{app:synthetic_details}
For the controlled experiments in Section~4.1, we generate data from a
stochastic Mackey--Glass system with additive state noise of magnitude
$\sigma_s$, using the standard delayed-feedback recursion from
\cite{mackey1977oscillation}. Unless otherwise stated, we use
$a=0.1$, $\gamma=0.2$, $\tau=17$, a burn-in of $100$ steps, total
length $T=6000$, lag $L=20$, horizon $H=200$, and an $80/20$
chronological split. For each selected test context, we approximate the
conditional distribution using $1000$ Monte Carlo rollouts

\subsection{Auxiliary realism metrics}
\label{app:realism_metrics}
Our primary realism metric is marginal $W_1$, because it is
hyperparameter free, horizon-local, and directly targets the
underdispersion mechanism central to the paper. As complementary
structural checks, we also report Dynamic Time Warping (DTW)
\cite{sakoe1978dynamic}, vector Wasserstein distance, order-preserving
Wasserstein (OPW) \cite{su2017order_wasser}, and fused
Gromov--Wasserstein (FGW) \cite{vayer2020fused}. These trajectory-level
metrics emphasize different aspects of temporal alignment, value
distribution, and geometric structure. We omit full equations here for space; our main empirical claims are in terms of marginal $W_1$.

We do not rely on real-world estimates of the conditional uncertainty
gap in the final model-selection analysis. Any finite-sample proxy for
the gap mixes the underlying signal with estimator bias and variance, so
we treat such quantities as explanatory diagnostics rather than
deployment-time requirements.

\section{Additional Robustness Checks}
\label{app:robustness}

\subsection{Stronger candidate model set}
\label{app:stronger_baselines}
To test whether the tolerance-band effect persists under a broader
candidate pool, we repeat the same $5\%$ MSE tolerance-band analysis
using an expanded set of forecasting models. The point-forecasting pool
includes Linear, MLP, LSTM, N-BEATS, PatchTST, and TCN models. The
probabilistic pool includes Linear, MLP, LSTM, normalizing flow,
PatchTST, and TCN variants.

For each horizon-specific evaluation, we identify the MSE optimal
candidate and then select the candidate with the best marginal realism
among all models whose MSE lies within $5\%$ of this optimum. Table
\ref{tab:stronger_baselines} reports the relative marginal-realism gain
of this near-MSE-optimal alternative, together with the percentage of
cases in which the realism gain exceeds $5\%$.

\begin{table}[t]
\centering
\setlength{\tabcolsep}{3pt}
\begin{tabular}{@{}lcccc@{}}
\toprule
Dataset & Mean & Med. & IQR & Gain$>5$ \\
\midrule
ExchangeRate & 21.7 & 11.7 & [3.9, 33.5] & 70.9 \\
PEMS08       & 15.1 & 12.5 & [5.9, 20.8] & 77.5 \\
Weather      & 12.5 & 8.8  & [3.2, 16.1] & 67.3 \\
ETTh1        & 18.5 & 17.1 & [7.8, 27.5] & 83.2 \\
MG           & 15.5 & 13.2 & [9.4, 18.0] & 94.1 \\
\bottomrule
\end{tabular}
\caption{Robustness of the $5\%$ MSE tolerance-band analysis under a
stronger candidate model set. Mean, median, and IQR denote relative
marginal-realism gains (\%) for the best-realism model inside the
$5\%$ MSE band compared with the MSE-optimal model. Gain$>5$ denotes
the percentage of cases in which the realism gain exceeds $5\%$.}
\label{tab:stronger_baselines}
\end{table}

The opportunity-cost pattern remains visible under the expanded model
set. Median marginal-realism gains are non-trivial across all evaluated
datasets, and the proportion of cases with realism gain above $5\%$
ranges from $67.3\%$ to $94.1\%$. This indicates that the practical
effect identified in the main text is not specific to the simpler model
families used in the main sweep.

\subsection{Probabilistic forecasts: calibration and width}
\label{app:calibration_width}
Marginal realism alone does not determine whether a probabilistic
forecast is conditionally well calibrated. A model could improve
marginal distributional agreement by producing less
informative wider predictive distributions. We complement the
marginal-realism analysis with coverage and interval-width diagnostics.

For each horizon-specific evaluation, we again consider the
best-realism candidate inside the $5\%$ MSE band and compare it with
the MSE-optimal candidate. Table~\ref{tab:calibration_width} reports
the median realism gain, the change in empirical $80\%$ coverage, and
the change in predictive interval width. Positive $\Delta C_{80}$
indicates improved agreement with nominal $80\%$ coverage. Positive
$\Delta I$ indicates wider intervals and hence lower sharpness.

\begin{table}[t]
\centering
\small
\setlength{\tabcolsep}{2.5pt}
\begin{tabular}{@{}lccccc@{}}
\toprule
Dataset
& Med. $\Delta R$
& Mean $\Delta C_{80}$
& Med. $\Delta C_{80}$
& Mean $\Delta I$
& Med. $\Delta I$ \\
\midrule
ExchangeRate & 23.1 & 1.306 & 0.000 & 14.3  & 26.3 \\
PEMS08       & 13.5 & 0.308 & 0.000 & -2.9  & -1.2 \\
Weather      & 9.6  & 0.209 & 0.017 & -0.03 & -0.00 \\
ETTh1        & 13.7 & 0.727 & 0.264 & 2.4   & 0.0 \\
MG           & 13.1 & 0.401 & 0.000 & -1.8  & 0.0 \\
\bottomrule
\end{tabular}
\caption{Calibration and interval-width diagnostics for the
best-realism probabilistic forecast inside the $5\%$ MSE band compared
with the MSE-optimal forecast. $\Delta R$ is the relative
marginal-realism gain (\%). Positive $\Delta C_{80}$ indicates improved
empirical $80\%$ coverage. $\Delta I$ is the relative change in
predictive interval width (\%); positive values indicate wider and hence
less sharp intervals.}
\label{tab:calibration_width}
\end{table}

The relation between marginal realism, calibration, and width is
heterogeneous. The realism-selected candidate is not typically worse
under the $80\%$ coverage diagnostic, and realism gains are not
uniformly obtained by widening predictive intervals. ExchangeRate shows
a clear increase in width, whereas PEMS08, Weather, and MG exhibit zero
or negative median width changes while still obtaining non-trivial
realism gains. This rules out the simplest universal explanation that
realism gains are always purchased through indiscriminate
over-dispersion.

\subsection{Robustness across alternative realism metrics}
\label{app:multi_metric_table}
Table~\ref{tab:multi_metric_opportunity_cost} reports the same
$\epsilon{=}5\%$ opportunity-cost analysis under multiple notions of
realism. For each dataset, variate, and horizon-specific evaluation, we
identify the MSE-optimal model and then consider Pareto-optimal
alternatives whose MSE is within $(1+\epsilon)\%$ of the optimum.
Within this near-tie set, we record the best attainable improvement in
realism for a given metric and the hit rate
$\Delta \mathrm{Realism} > \Delta \mathrm{MSE}$.

\begin{table}[t]
    \centering
    \small
    \setlength{\tabcolsep}{3pt}
    \begin{tabular}{lcccc}
        \toprule
        Metric & $\widetilde{\Delta R}$ & $\mathrm{IQR}(\Delta R)$ & $\widetilde{p}_{\mathrm{hit}}$ & $\mathrm{IQR}(p_{\mathrm{hit}})$ \\
        \midrule
        $W_1$      & 0.17 & [0.10,\,0.28] & 0.31 & [0.25,\,0.32] \\
        DTW        & 0.07 & [0.06,\,0.10] & 0.10 & [0.08,\,0.12] \\
        $W_1$-vec  & 0.08 & [0.06,\,0.10] & 0.15 & [0.11,\,0.16] \\
        OPW        & 0.09 & [0.07,\,0.12] & 0.04 & [0.01,\,0.10] \\
        FGW        & 0.11 & [0.07,\,0.17] & 0.19 & [0.16,\,0.23] \\
        \bottomrule
    \end{tabular}
    \caption{Multi-metric opportunity cost under an MSE tolerance of
    $\epsilon=5\%$. For each realism metric, we report the median across
    datasets of the dataset-level median realism gain
    ($\widetilde{\Delta R}$) and the dataset-level hit rate
    ($\widetilde{p}_{\mathrm{hit}}$), together with the corresponding
    interquartile ranges across datasets. ``Hit'' denotes
    $\Delta \mathrm{Realism} > \Delta \mathrm{MSE}$.}
    \label{tab:multi_metric_opportunity_cost}
\end{table}

The opportunity-cost effect remains visible beyond marginal $W_1$, but
its magnitude depends strongly on the realism notion used. Marginal
$W_1$ yields the largest gains and highest hit rates because it targets
the horizon-wise dispersion mismatch central to our theory, whereas
trajectory metrics impose stronger structural correspondence and thus
typically require larger departures from the MSE optimum.

\subsection{What different realism metrics detect}
\label{app:what_each_metric}

A useful realism metric should reflect the kind of mismatch one aims to
detect. In our setting, the primary failure mode is
\emph{under-dispersion}: the MSE-optimal conditional expectation can be
accurate on average while becoming unrepresentative of typical realized
futures. Under a single-realization evaluation protocol, this suggests
three desiderata for a diagnostic metric. First, two independent samples
from the same data-generating process should be judged similar. Second,
the conditional expectation should move further away from a realized
sample as conditional uncertainty grows. Third, if temporal structure is
part of the realism notion, then a sample should be distinguishable from
a within-trajectory permutation of itself.

\begin{figure}[t]
    \centering
    \includegraphics[width=\linewidth]{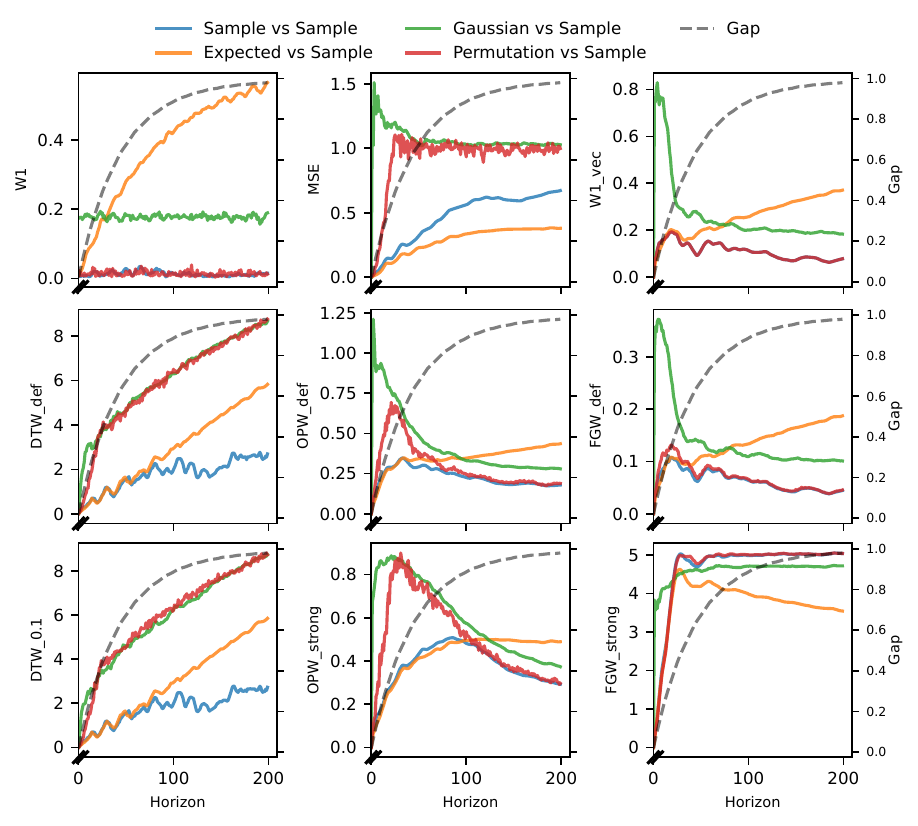}

    \caption{\textbf{Metric audit on the stochastic Mackey--Glass
    system.} We compare four deterministic sequences against an
    independent DGP sample: sample vs sample, conditional
    expectation vs sample, variance matched Gaussian vs sample, and
    permutation vs sample. The diagnostic asks whether a metric
    (i) keeps sample-vs-sample small, (ii) makes
    expectation-vs-sample grow with uncertainty, and
    (iii) penalizes permutations when temporal structure is broken.
    Marginal $W_1$ satisfies (i) and (ii) cleanly, but not (iii), since it ignores order. DTW, OPW, FGW, and $W_{1,\mathrm{vec}}$ are more sensitive to temporal structure, 
    but also respond more strongly to alignment and point-wise deviations.}
    
    \label{fig:metric_audit_controlled_baselines}
    \Description[Metric audit compares sensitivity to uncertainty and temporal order]{Three-by-three grid of line plots comparing realism metrics on stochastic Mackey--Glass sequences. Rows correspond to different metrics and columns compare sample, expectation, Gaussian, and permutation baselines against independent samples. The figure shows that marginal Wasserstein distance clearly detects under-dispersion but ignores temporal ordering, while trajectory metrics are more sensitive to ordering and alignment.}
\end{figure}

Figure~\ref{fig:metric_audit_controlled_baselines} shows that marginal
$W_1$ is the clearest diagnostic for the phenomenon studied here. It
keeps sample-vs-sample near zero while making
expectation-vs-sample grow with uncertainty, directly exposing the
under-dispersion of the conditional mean. Its limitation is that
permutation-vs-sample is also small, since marginal $W_1$ ignores
temporal order.

DTW, OPW, FGW, and $W_{1,\mathrm{vec}}$ attempt to encode stronger notions of
trajectory similarity. However, this makes them more sensitive to ordering and
shape distortions, but also more correlated with pointwise error, which
helps explain their smaller opportunity-cost gains. We therefore use
marginal $W_1$ as the paper's primary realism axis and treat the trajectory metrics as complementary checks.

\subsection{Zero-shot Chronos robustness}
\label{app:chronos}

As an additional robustness check using a strong forecasting
foundation model, we evaluate Amazon Chronos-T5
\cite{ansari2024chronos} in zero-shot mode on the same real-world
benchmarks. We use the same dataset splits as in the main experiments,
standardize each series independently, and construct rolling
forecasting instances with lookback $512$ and horizons up to $50$,
following the operating range recommended by the Chronos authors.

We evaluate three Chronos sizes
(\texttt{chronos-t5-tiny}, \\ \texttt{chronos-t5-mini},
\texttt{chronos-t5-small}) across sampling temperatures
$\tau\in\{0.6,0.8,1.0,1.2,1.4\}$. For each context we draw $100$
forecast samples and form operating points using either individual
sampled trajectories or means of multiple samples. As in the main
paper, sample-based operating points occupy the realism-favored region,
whereas mean-of-samples forecasts concentrate near the MSE-favored
extreme.

\begin{figure}[t]
    \centering
    \includegraphics[width=\linewidth]{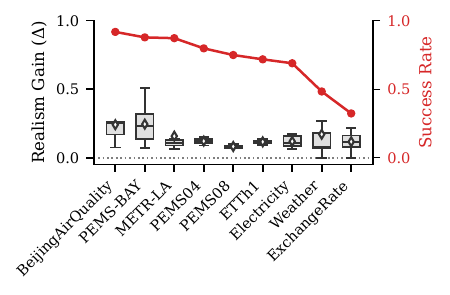}

    \caption{Reproduction of the main-text $\epsilon$-opportunity
    summary using zero-shot Chronos-T5 (lookback $512$, $H\le 50$),
    following the settings recommended in \cite{ansari2024chronos}.}
    \label{fig:chronos}
    \Description[Chronos results reproduce the opportunity-cost pattern]{Boxplots summarizing zero-shot Chronos-T5 results under the same five percent MSE tolerance analysis. Each dataset shows attainable marginal realism gains, with red markers indicating success rates. The figure shows that the accuracy--realism opportunity pattern remains visible for Chronos forecasts.}
\end{figure}

The same $\epsilon$-opportunity pattern remains visible under Chronos.
Across datasets we observe consistently positive median realism gains
together with generally high hit rates. The strongest results are on
BeijingAirQuality ($0.92$ hit rate; median/mean gain $0.25/0.24$),
PEMS-BAY ($0.88$; $0.23/0.24$), and METR-LA \\($0.87$; $0.10/0.16$).
We also observe positive median gains on PEMS04 ($0.80$;
$0.12/0.12$), PEMS08 ($0.75$; $0.08/0.08$), ETTh1 ($0.72$;
$0.11/0.12$), Electricity ($0.69$; $0.11/0.12$), Weather ($0.48$;
$0.08/0.17$), and ExchangeRate ($0.32$; $0.11/0.12$). Thus, the
practical opportunity cost of strict MSE-only selection is also
visible for a strong zero-shot foundation model.

\section{Additional Theoretical Results}
\label{app:additional_theory}

This section records short extensions of
Theorem~\ref{thm:tradeoff}.

\subsection{Why determinism matters}
\label{sec:stochastic_forecasters}
Theorem~\ref{thm:tradeoff} is stated for deterministic point
predictors, which is the standard setting for MSE-trained multi-step
forecasters. The key fact is that stochastic randomization cannot
improve expected MSE beyond the conditional mean.

\begin{proposition}[Randomization cannot improve MSE]
\label{prop:rand_mse}
Let $\widehat{Y}_h$ be any (possibly randomized) predictor given
$\mathbf{X}$ with finite second moment. Then
\[
\mathbb{E}\!\left[(\widehat{Y}_h-Y_h)^2\right]
\ge
\mathbb{E}\!\left[(\mathbb{E}[Y_h\mid \mathbf{X}]-Y_h)^2\right],
\]
with equality if and only if
$\widehat{Y}_h=\mathbb{E}[Y_h\mid \mathbf{X}]$ almost surely.
\end{proposition}

\noindent\textit{Proof sketch.}
Conditioning on $\mathbf{X}$ gives
\[
\mathbb{E}\!\left[(\widehat{Y}_h-Y_h)^2 \mid \mathbf{X}\right]
=
\mathbb{E}\!\left[(\widehat{Y}_h-\mathbb{E}[Y_h\mid \mathbf{X}])^2
\mid \mathbf{X}\right]
+
\mathrm{Var}(Y_h\mid \mathbf{X}),
\]
where the cross-term vanishes. Taking expectation over $\mathbf{X}$
yields the result. Thus, any stochasticity that helps match a
distributional target must be paid for in expected MSE.

\subsection{Single-context trajectory mismatch}
\label{sec:trajectory_dirac}
Theorem~\ref{thm:tradeoff} compares marginal laws after averaging over
contexts. A complementary statement holds at the level of conditional
trajectory laws for a fixed context.

Let $\mathbf{Y}_{:H}=(Y_1,\ldots,Y_H)\in\mathbb{R}^H$ denote the
future trajectory and define
\[
P_x := \mathcal{L}(\mathbf{Y}_{:H}\mid \mathbf{X}=x), \qquad
\mathbf{m}(x) := \mathbb{E}[\mathbf{Y}_{:H}\mid \mathbf{X}=x].
\]
The Bayes-optimal deterministic predictor under vector MSE is
$\mathbf{m}(x)$. Because this predictor is deterministic, its
conditional predictive law is the Dirac measure
$\delta_{\mathbf{m}(x)}$.


\begin{proposition}[Conditional trajectory mismatch]
\label{prop:traj_w2}
Assume $\mathbb{E}\|\mathbf{Y}_{:H}\|^2<\infty$. Then for any fixed
context $x$,
\begin{align*}
W_2^2\!\left(P_x,\delta_{\mathbf{m}(x)}\right)
&=
\mathbb{E}\!\left[\left\|\mathbf{Y}_{:H}-\mathbf{m}(x)\right\|^2
\,\middle|\, \mathbf{X}=x\right] \\
&=
\mathrm{tr}\!\left(\mathrm{Cov}(\mathbf{Y}_{:H}\mid \mathbf{X}=x)\right).
\end{align*}
In particular, if
$\mathrm{tr}(\mathrm{Cov}(\mathbf{Y}_{:H}\mid \mathbf{X}=x))>0$,
then $W_2(P_x,\delta_{\mathbf{m}(x)})>0$.
\end{proposition}

\noindent\textit{Proof sketch.}
The only coupling between $P_x$ and a Dirac mass $\delta_a$ pairs
$\mathbf{Y}_{:H}$ with the constant vector $a$, so
$W_2^2(P_x,\delta_a)=\mathbb{E}[\|\mathbf{Y}_{:H}-a\|^2\mid
\mathbf{X}=x]$. Minimizing over $a$ gives $a=\mathbf{m}(x)$, and the
result follows from the standard identity between expected squared
deviation and the trace of the conditional covariance matrix.

\paragraph{Implication for trajectory-level law metrics.}
Proposition~\ref{prop:traj_w2} is stated with $W_2$, but the mismatch is
not specific to Wasserstein distance. Whenever $P_x$ is non-degenerate,
we have $P_x \neq \delta_{\mathbf{m}(x)}$, so any trajectory-level
metric or divergence on \emph{probability laws over trajectories} that
vanishes only when the two laws are equal must also be strictly positive
at $(P_x,\delta_{\mathbf{m}(x)})$.
Moreover, equality of
trajectory laws implies equality of all horizon-wise marginals by
coordinate projection; hence a mismatch at any marginal horizon already
rules out any stronger trajectory-law match.
This statement does not automatically extend to arbitrary
instance-level sequence distances such as DTW between two sampled
trajectories, which compare individual realizations rather than
conditional laws.

\end{document}